\documentclass{llncs}
\usepackage{llncsdoc}
\usepackage[english]{babel}
\usepackage{latexsym}
\usepackage{mathrsfs}
\usepackage{proof}
\usepackage{bussproofs}
\usepackage{hyperref}
\usepackage{stmaryrd}
\usepackage{multirow}

\newcommand{\tseq}[2]{\begin{tabular}{|l|l|}\hline $#1$ & $#2$ \\ \hline\end{tabular}}

\newcommand{\np}{\textit{np}}

\newcommand{\bs}{\backslash}
\newcommand{\os}{\varoslash}
\newcommand{\obs}{\varobslash}
\newcommand{\lgtimes}{\varotimes}
\newcommand{\lgplus}{\varoplus}

\spnewtheorem{thm}[theorem]{Theorem}{\bfseries}{\rm}
\spnewtheorem{cor}[theorem]{Corollary}{\bfseries}{\rm}
\spnewtheorem{vb}[theorem]{Example}{\bfseries}{\rm}
\spnewtheorem{lem}[theorem]{Lemma}{\bfseries}{\rm}
\spnewtheorem{rem}[theorem]{Remark}{\bfseries}{\rm}
\spnewtheorem{defs}[theorem]{Definition}{\bfseries}{\rm}

\title{Tableaux for the Lambek-Grishin Calculus}
\author{Arno Bastenhof}
\institute{Utrecht University}

\begin{document}
\bibliographystyle{plain}
\maketitle
Categorial type logics, pioneered by Lambek (\cite{lambek58}), seek a proof-theoretic understanding of natural language syntax by identifying categories with formulas and derivations with proofs. We typically observe an intuitionistic bias: a structural configuration of hypotheses (a constituent) derives a single conclusion (the category assigned to it). Acting upon suggestions of Grishin (\cite{grishin}) to dualize the logical vocabulary, Moortgat proposed the \textit{Lambek-Grishin calculus} (\textbf{LG}, \cite{moortgat09}) with the aim of restoring symmetry between hypotheses and conclusions. 

We propose a theory of \textit{labeled} modal tableaux (\cite{smullyan}) for \textbf{LG}, inspired by the interpretation of its connectives as binary modal operators in the relational semantics of \cite{kurtoninamoortgat10}. After a brief recapitulation of \textbf{LG}'s models in $\S$1, we define our tableaux in $\S$2 and ensure soundness and completeness in $\S$3. Linguistic applications are considered in $\S$4, where grammars based on \textbf{LG} are shown to be context-free through use of an interpolation lemma. This result complements \cite{melissenfg09}, where \textbf{LG} augmented by mixed associativity and -commutativity was shown to exceed LTAG in expressive power.

\section{Ternary frames and Lambek calculi}
We discuss ternary frame semantics for \textbf{NL} and its symmetric generalization. More in-depth discussions of the presented material is found in \cite{kurtonina95} and \cite{kurtoninamoortgat10}. 

\textit{Ternary frames} $\mathscr{F}$ are pairs $\langle W,R\rangle$ with $W$ an inhabited set of \textit{resources}, and $R\subseteq W^3$ a \textit{(ternary) accessibility relation}. Propositional variables (atoms) $p,q,r,\dots$ are identified with subsets of $W$. Formally, a \textit{model} $\mathscr{M}=\langle\mathscr{F},V\rangle$ extends $\mathscr{F}$ with an \textit{(atomic) valuation} $V$ mapping atoms to $\mathscr{P}(W)$. Connectives for (multiplicative) conjunction and implication, constructing derived formulas $A,B,C,\dots$, arise as \textit{binary modal operators} by extending $V$ as in
\begin{center}
\begin{tabular}{rcll}
$V(A\lgtimes B)$ & $:=$ & $\{ x \ | \ (\exists y,z)(Rxyz \ \textrm{and} \ y\in V(A) \ \textrm{and} \ z\in V(B))\}$ & \textit{(fusion)} \\
$V(C/B)$ & $:=$ & $\{ y \ | \ (\forall x,z)((Rxyz \ \textrm{and} \ z\in V(B))\Rightarrow x\in V(C))\}$ & \textit{(right impl.)} \\
$V(A\bs C)$ & $:=$ & $\{ z \ | \ (\forall x,y)((Rxyz \ \textrm{and} \ y\in V(A))\Rightarrow x\in V(C))\}$ & \textit{(left impl.)}
\end{tabular}
\end{center}
Kurtonina (\cite{kurtonina95}) explores linguistic applications. $W$ contains syntactic constituents and $Rxyz$ reads as binary merger: $x$ results from merging $y$ with $z$. Thus, one would adopt atoms $\np$ (its image under $V$ the collection of noun phrases), $s$ (sentences) and $n$ (common nouns), with subcategorization encoded by implications: $\np\bs s$ categorizes intransitive verbs, $(\np\bs s)/\np$ transitive verbs, etc.

\textit{Proofs} (or \textit{algebraic derivations}) of inequalities $A\leq B$ are intended to establish $V(A)\subseteq V(B)$ for arbitrary $\langle\mathscr{F},V\rangle$. On the linguistic reading, these are the \textit{language universals}: any language categorizing an expression by $A$ (e.g., $\np\lgtimes(\np\bs s)$, the merger of a noun phrase and an intransitive verb) must also categorize it by $B$ ($s$, as follows from the rules below). 
Next to the preorder axioms \textit{(Refl, Trans)} on $\leq$, the set $\{\lgtimes,/,\bs\}$ is \textit{residuated} $(r)$, with \textit{parent} $\lgtimes$ and (left and right) \textit{residuals} $\bs,/$ (the double line indicates \textit{inter}derivability)
\begin{center}
\begin{tabular}{ccccccc}
$\infer[\textit{Refl}]{A\leq A}{}$ \ \ & & 
 \ \ $\infer[\textit{Trans}]{A\leq C}{A\leq B & B\leq C}$ \ \ & &
 \ \ $\infer=[r]{B\leq A\bs C}{A\lgtimes B\leq C}$ \ & & 
 \ \ $\infer=[r]{A\leq C/B}{A\lgtimes B\leq C}$
\end{tabular}
\end{center}
validity w.r.t. arbitrary models being easily verified. The distinguished status of fusion leads us to write the corresponding accessibility relation as $R_{\lgtimes}$ from now on. We arrive at what is known as the \textit{non-associative Lambek calculus} (\textbf{NL}). Note that associativity and commutativity of $\lgtimes$ are not generally valid in a model, but rather depend on special frame constraints (cf. \cite{kurtonina95}):
\begin{center}
\begin{tabular}{|rcl|l|} \hline
\multicolumn{3}{|l|}{\textbf{Inequality}} & \textbf{Frame constraint} $(\forall a,b,c,y,x)$ \\ \hline
$A\lgtimes(B\lgtimes C)$ & $\leq$ & $(A\lgtimes B)\lgtimes C$ & $(R_{\lgtimes}xay \ \textrm{and} \ R_{\lgtimes}ybc)\Rightarrow (\exists t)(R_{\lgtimes}xtc \ \textrm{and} \ R_{\lgtimes}tab)$ \\
$(A\lgtimes B)\lgtimes C$ & $\leq$ & $A\lgtimes(B\lgtimes C)$ & $(R_{\lgtimes}xtc \ \textrm{and} \ R_{\lgtimes}tab)\Rightarrow (\exists x)(R_{\lgtimes}xay \ \textrm{and} \ R_{\lgtimes}ybc)$ \\
$A\lgtimes B$ & $\leq$ & $B\lgtimes A$ & $R_{\lgtimes}xab\Rightarrow R_{\lgtimes}xba$ \\ \hline
\end{tabular}
\end{center}
Grishin (\cite{grishin}) first suggested extending \textbf{NL} by a family of \textit{co}residuated connectives $\{\lgplus,\os,\obs\}$ with parent $\lgplus$ (\textit{fission}) and left- and right coresiduals $\obs,\os$ (\textit{subtractions}), mirroring $\{\lgtimes,/,\bs\}$ in $\leq$:
\begin{center}
\begin{tabular}{ccc}
$\infer=[\textit{cr}]{C\os B\leq A}{C\leq A\lgplus B}$ \ \ & &
 \ \ $\infer=[\textit{cr}]{A\obs C\leq B}{C\leq A\lgplus B}$
\end{tabular}
\end{center}
Moortgat names this the \textit{Lambek-Grishin calculus} (\textbf{LG}) in \cite{moortgat09}. In contrast with \textit{classical} \textbf{NL} (\cite{degrootelamarche}), \textbf{LG} does not internalize its duality with linear negation. Thus, we cannot simply interpret fission and subtraction as the De Morgan duals of fusion and implication. Instead, we have to consider frames $\mathscr{F}=\langle W,R_{\lgtimes},R_{\lgplus}\rangle$ with a second accessibility relation $R_{\lgplus}\subseteq W^3$:
\begin{center}
\begin{tabular}{rcl}
$x\in V(A\lgplus B)$ & $\Leftrightarrow$ & $(\forall y,z)(R_{\lgplus}xyz\Rightarrow(y\in V(A) \ \textrm{or} \ z\in V(B)))$ \\
$y\in V(C\os B)$ & $\Leftrightarrow$ & $(\exists x,z)(R_{\lgplus}xyz \ \textrm{and} \ z\in V(B) \ \textrm{and} \ x\in V(C))$ \\
$z\in V(A\obs C)$ & $\Leftrightarrow$ & $(\exists x,y)(R_{\lgplus}xyz \ \textrm{and} \ y\in V(A) \ \textrm{and} \ x\in V(C))$
\end{tabular}
\end{center}
We conclude by mentioning previous work on the proof theory of \textbf{LG}, motivating our own tableau approach. First, a negative result: while Lambek (\cite{lambek61}) gave a sequent calculus for \textbf{NL}, extending it to \textbf{LG} by mirroring the inference rules sacrifices Cut admissibility.\footnote{Bernardi and Moortgat give a(n unpublished) counterexample with the two-formula sequent $A\lgtimes(C\os((A\bs B)\obs C))\vdash B$.} Moortgat (\cite{moortgat09}) instead defines a \textit{display calculus} for \textbf{LG}, based on the observation that (algebraic) transitivity is admissible in the presence of (co)residuation and monotonicity:
\begin{center}
\begin{tabular}{ccccc}
$\infer{\begin{array}{l}A\lgtimes C\leq B\lgtimes D \\ A\lgplus C\leq B\lgplus D\end{array}}{A\leq B & C\leq D}$ & \ \ &
$\infer{\begin{array}{l}A/D\leq B/C \\ D\bs A\leq C\bs B\end{array}}{A\leq B & C\leq D}$ & \ \ &
$\infer{\begin{array}{l}A\os D\leq B\os C \\ D\obs A\leq C\obs B\end{array}}{A\leq B & C\leq D}$
\end{tabular}
\end{center}
Our own approach to \textbf{LG} theorem proving is rather in the tradition of labeled modal tableaux, mixing the language of formulas with that of the models interpreting them. Equivalently, the old ``turn your derivations upside-down'' trick renders it as a \textit{labeled} sequent calculus, representing by a single labeled sequent those display sequents of \cite{moortgat09} that are interderivable by (co)residuation. Moreover, as Lemma \ref{cutelim} shows, Cut-admissibility is recovered. 

\section{A labeled tableau calculus for LG}
Fix a denumerable collection of variables $x,y,z,\dots$, to be thought of as a set $W$ of resources. By a \textit{signed formula} we understand a formula suffixed by $\cdot^{\bullet}$ or $\cdot^{\circ}$. We also speak of \textit{input} formulas $A^{\bullet}$ and \textit{output} formulas $A^{\circ}$. A \textit{labeled} signed formula pairs a signed formula with a variable. Intuitively, a pair $x:A^{\bullet}$ asserts $A$ to be true at point $x$, whereas $y:B^{\circ}$ asserts $B$ to be false at point $y$. We sometimes use  meta-variables $\phi,\psi,\omega$, using the suffix $\cdot^{\bot}$ for switching signs: $\phi^{\bot}$ denotes $x:A^{\circ}$ if $\phi=x:A^{\bullet}$ and $x:A^{\bullet}$ if $\phi=x:A^{\circ}$.

Tableau rules operate on \textit{boxes} $\tseq{\Theta}{\Gamma}$, understood linguistically as encoding syntactic descriptions: phrase structure is specified by means of an \textit{unrooted tree} $\Theta$, with $\Gamma$ defining a cyclic order on the words attached to its leaves. More specifically, $\Gamma$ denotes a finite list of signed formulas (categorizing words) labeled by variables found at the leaves of $\Theta$, such that 'provability' of a box $\tseq{\Theta}{\Gamma}$ will be closed under cyclic permutations of $\Gamma$. We describe trees $\Theta$ by multisets of conditions $R_{\lgtimes}xyz$, $R_{\lgplus}xyz$: each variable in (a condition of) $\Theta$ has its own node, any condition $R_{\lgplus}xyz$ or $R_{\lgtimes}xyz$ in $\Theta$ introduces a fresh node with edges (precisely) to $x,y,z$, and any variable occurs at most twice.\footnote{Such structures previously appeared in the literature on (non-associative) proof nets as \textit{tensor trees} in \cite{moot07} and as \textit{tree signatures} in \cite{lamarche03}, the latter building forth on \cite{degrootelamarche}. In \cite{kurtonina95}, a similar encoding of rooted trees by means of the accessibility relation $R_{\lgtimes}xyz$ was proposed for \textbf{NL}.}
\begin{center}
\begin{tabular}{|l|l|l|l|l|} \hline
\textbf{Trees} $\Theta$ & $x$ & $R_{\lgtimes}xyz$ & $R_{\lgplus}xyz$ & $(\Theta'\bigcup\Theta'')-(N(\Theta')\bigcup N(\Theta''))$ \\ \hline\hline
\multirow{2}{*}{\textbf{Conditions}} & - & - & - & $N(\Theta')\bigcap N(\Theta'')=\{ x\}$ \\
 & & & & $x\in C(\Theta'), x\in H(\Theta'')$ \\ \hline
\textbf{Nodes} $N(\Theta)$ & $\{ x\}$ & $\{ x,y,z\}$ & $\{ x,y,z\}$ & $N(\Theta')\bigcup N(\Theta'')$ \\ \hline
\textbf{Hypotheses} $H(\Theta)$ & $\{ x\}$ & $\{ y,z\}$ & $\{ x\}$ & $(H(\Theta')\bigcup H(\Theta''))/\{ x\}$ \\ \hline
\textbf{Conclusions} $C(\Theta)$ & $\{ x\}$ & $\{ x\}$ & $\{ y,z\}$ & $(C(\Theta')\bigcup C(\Theta''))/\{ x\}$ \\ \hline
\end{tabular}
\end{center}
Thus, for any such 'tree' $\Theta$, $\tseq{\Theta}{\Gamma}$ is a box in case the hypotheses of $\Theta$ label input formulas of $\Gamma$, whereas its conclusions label output formulas. Note that our use of multiset difference in the definition of complex trees implies $x\in\Theta$ only if $\Theta$ is a singleton. The purpose of such trees $\{ x\}$ is to guarantee well-definedness for the concepts $N(\Theta)$, $H(\Theta)$ and $C(\Theta)$ w.r.t. two-formula boxes $\tseq{x}{x:A^{\bullet},x:B^{\circ}}$. We shall often abbreviate $(\Theta\bigcup\Theta')-(N(\Theta)\bigcup N(\Theta'))$ by $\Theta,\Theta'$ (in particular: $\Theta,\{ x\}=\Theta$), and similarly write $\Gamma,\Delta$ for list concatenation. 

Labeled signed formulas are classified into types $\alpha,\beta$ according to Smullyan's unified notation:
\begin{center}
\begin{tabular}{|c|ccc||c|ccc|} \hline
$\alpha$ & $\alpha_1(y)$ \ \ & \ \ $\alpha_2(z)$ \ \ & \ \ $R_{\alpha}yz$ \ \ & $\beta$ & $\beta_1(y)$ \ \ & \ \ $\beta_2(z)$ \ \ & \ \ $R_{\beta}yz$ \\ \hline
$x:(A/B)^{\circ}$ & $y:B^{\bullet}$ & $z:A^{\circ}$ & $R_{\lgtimes}zxy$ & $x:(A/B)^{\bullet}$ & $y:B^{\circ}$ & $z:A^{\bullet}$ & $R_{\lgtimes}zxy$ \\
$x:(B\bs A)^{\circ}$ & $y:A^{\circ}$ & $z:B^{\bullet}$ & $R_{\lgtimes}zyx$ & $x:(B\bs A)^{\bullet}$ & $y:A^{\bullet}$ & $z:B^{\circ}$ & $R_{\lgtimes}zyx$ \\
$x:(A\lgtimes B)^{\bullet}$ & $y:A^{\bullet}$ & $z:B^{\bullet}$ & $R_{\lgtimes}xyz$ & $x:(A\lgtimes B)^{\circ}$ & $y:A^{\circ}$ & $z:B^{\circ}$ & $R_{\lgtimes}xyz$ \\
$x:(A\os B)^{\bullet}$ & $y:A^{\bullet}$ & $z:B^{\circ}$ & $R_{\lgplus}zxy$ & $x:(A\os B)^{\circ}$ & $y:A^{\circ}$ & $z:B^{\bullet}$ & $R_{\lgplus}zxy$ \\
$x:(B\obs A)^{\bullet}$ & $y:B^{\circ}$ & $z:A^{\bullet}$ & $R_{\lgplus}zyx$ & $x:(B\obs A)^{\circ}$ & $y:B^{\bullet}$ & $z:A^{\circ}$ & $R_{\lgplus}zyx$ \\  
$x:(A\lgplus B)^{\circ}$ & $y:B^{\circ}$ & $z:A^{\circ}$ & $R_{\lgplus}xyz$ & $x:(A\lgplus B)^{\bullet}$ & $y:B^{\bullet}$ & $z:A^{\bullet}$ & $R_{\lgplus}xyz$ \\ \hline
\end{tabular}
\end{center}
Tableaux may then be expanded by either one of the following rules, of which the second is said to \textit{branch}:\footnote{The current formulation may be considered a labeling of Abrusci's sequent calculus for cyclic linear logic in \cite{abrusci02}, where cyclic permutations were compiled away into the logical inferences.}
\begin{center}
\begin{tabular}{ccc}
$\infer[\alpha]{\tseq{R_{\alpha}yz,\Theta}{\Gamma,\alpha_1(y),\alpha_2(z),\Delta}}{
\tseq{\Theta}{\Gamma,\alpha,\Delta}}$ & \ \ \ \ \ &
$\infer[\beta]{\tseq{\Theta}{\Gamma,\beta_1(y),\Gamma'} \ \ \ \ \tseq{\Theta'}{\Delta,\beta_2(z),\Delta'}}{
\tseq{R_{\beta}yz,\Theta,\Theta'}{\Gamma,\Delta,\beta,\Gamma',\Delta'}}$
\end{tabular}
\end{center}
Here, in $(\alpha)$, $y,z$ are to be fresh in the current branch, whereas for $(\beta)$, either $\Gamma=\emptyset$ or $\Delta'=\emptyset$, and
\begin{center}
\begin{tabular}{ccc}
\begin{tabular}{l}
$N(\Theta)\bigcap N(\Theta')=\emptyset$, \\
$N(R_{\beta}yz)\bigcap N(\Theta)=\{ z\}$, \\
$N(R_{\beta}yz)\bigcap N(\Theta')=\{ y\}$ 
\end{tabular} & \ \ \ \ and \ \ \ \ &
\begin{tabular}{|l|l|l|} \hline
 & $\in H(R_{\beta}yz),$ & $\in C(R_{\beta}yz),$ \\ \hline
If $y$ & then $y\in C(\Theta')$ & then $y\in H(\Theta')$ \\
If $z$ & then $z\in H(\Theta)$ & then $z\in C(\Theta)$ \\ \hline
\end{tabular}
\end{tabular}
\end{center}
A tableau branch ending in a box $\tseq{x}{x:p^{\bullet},x:p^{\circ}}$ or $\tseq{x}{x:p^{\circ},x:p^{\bullet}}$ is \textit{closed}, and a tableau is closed if all its branches are. We also say $\tseq{\Theta}{\Gamma}$ closes if it has a closed tableau. A tableau for a two-formula \textit{sequent} $A\vdash B$ is a tableau of $\tseq{x}{x:A^{\bullet},x:B^{\circ}}$, called a \textit{proof} of $A\vdash B$ if it closes. An easy induction establishes
\begin{lem}\label{rename}
The property of having a closed tableau (of a box) is preserved under renaming of variables.
\end{lem}
\noindent Say a tableau of $\tseq{\Theta}{\Gamma}$ \textit{closes via $\phi\in\Gamma$} if the first expansion immediately targets $\phi$, and let the \textit{degree} of a formula $A$ denote the number of connectives in $A$.
\begin{lem}\label{cycl} \textit{Cyclic permutation} is admissible: if we have a closed tableau $\mathscr{T}$ of $\tseq{\Theta}{\Gamma,\Delta}$, then $\tseq{\Theta}{\Delta,\Gamma}$ also closes.
\end{lem}
\begin{proof}
By induction on the (combined) degree of (the formulas in) $\Gamma,\Delta$. 
\begin{enumerate}
\item $\tseq{\Theta}{\Gamma,\Delta}=\tseq{x}{x:p^{\bullet},x:p^{\circ}}$ or $\tseq{\Theta}{\Gamma,\Delta}=\tseq{x}{x:p^{\circ},x:p^{\bullet}}$: immediate.
\item $\mathscr{T}$ closes via some $\alpha\in\Gamma,\Delta$. Say $\alpha\in\Gamma$, i.e., $\Gamma=\Gamma',\alpha,\Gamma''$:
\begin{prooftree}
\rootAtTop
\AxiomC{\tseq{R_{\alpha}yz,\Theta}{\Gamma',\alpha_1(y),\alpha_2(z),\Gamma'',\Delta}}
\RightLabel{$\alpha$}
\UnaryInfC{\tseq{\Theta}{\Gamma',\alpha,\Gamma'',\Delta}}
\end{prooftree}
Then, by the induction hypothesis, $\tseq{R_{\alpha}yz,\Theta}{\Delta,\Gamma',\alpha_1(y),\alpha_2(z),\Gamma''}$ also closes, so that the statement of the lemma now obtains by another $\alpha$-expansion.

\item $\mathscr{T}$ closes via some $\beta\in\Gamma,\Delta$. I.e., $\Gamma,\Delta=\Gamma_1,\Gamma_2,\beta,\Gamma_1',\Gamma_2'$ and $\Theta=R_{\beta}yz,\Theta,\Theta'$, with, for example, $\Gamma_2'=\Delta_1,\Delta_1'$, $\Gamma=\Gamma_1,\Gamma_2,\beta,\Gamma_1',\Delta_1$ and $\Delta=\Delta_1'$: 
\begin{prooftree}
\rootAtTop
\AxiomC{\tseq{\Theta}{\Gamma_1,\beta_1(y),\Gamma_1'}}
\AxiomC{\tseq{\Theta'}{\Gamma_2,\beta_2(z),\Delta_1,\Delta_1'}}
\RightLabel{$\beta$}
\BinaryInfC{\tseq{R_{\beta}yz,\Theta,\Theta'}{\Gamma_1,\Gamma_2,\beta,\Gamma_1',\Delta_1,\Delta_1'}}
\end{prooftree}
By the induction hypothesis, $\tseq{\Theta'}{\Delta_1',\Gamma_2,\beta_2(z),\Delta_1}$ closes, so that another $\beta$-expansion suffices:
\begin{prooftree}
\rootAtTop
\AxiomC{\tseq{\Theta}{\Gamma_1,\beta_1(y),\Gamma_1'}}
\AxiomC{\tseq{\Theta'}{\Delta_1',\Gamma_2,\beta_2(z),\Delta_1}}
\RightLabel{$\beta$}
\BinaryInfC{\tseq{R_{\beta}yz,\Theta,\Theta'}{\Gamma_1,\Delta_1',\Gamma_2,\beta,\Gamma_1',\Delta_1}}
\end{prooftree}
noting that if $\Gamma_1\not =\emptyset$, then $\Delta(=\Delta_1,\Delta_1')=\emptyset$. 
\end{enumerate}
\end{proof}

\begin{vb}\label{example_tableau} 
We have a proof of $p\lgtimes(r\os((p\bs q)\obs r))\vdash q$, which served as a counterexample to Cut elimination in an earlier sequent calculus for \textbf{LG}.
\begin{prooftree}
\rootAtTop
\AxiomC{\tseq{x}{x:q^{\bullet},x:q^{\circ}}}
\AxiomC{\tseq{y}{y:p^{\bullet},y:p^{\circ}}}
\RightLabel{$\beta$}
\BinaryInfC{\tseq{R_{\lgtimes}xyz}{y:p^{\bullet},z:(p\bs q)^{\bullet},x:q^{\circ}}}
\AxiomC{\tseq{u}{u:r^{\bullet},u:r^{\circ}}}
\RightLabel{$\beta$}
\BinaryInfC{\tseq{R_{\lgplus}uzv,R_{\lgtimes}xyz}{y:p^{\bullet},u:r^{\bullet},v:((p\bs q)\obs r)^{\circ},x:q^{\circ}}}
\RightLabel{$\alpha(\times 2)$}
\UnaryInfC{\tseq{x}{x:(p\lgtimes(r\os((p\bs q)\obs r)))^{\bullet},x:q^{\circ}}}
\end{prooftree}
\end{vb}

\begin{vb}\label{ling_ex1}
We have previously understood boxes as encodings for syntactic descriptions. We further illustrate this claim by representing the derivation of a simple transitive clause by a closed tableau. Consider the following lexicon for \textit{He saw Pete}, consisting of a pairing of words with signed formulas:
\begin{center}
\begin{tabular}{|ccc|}\hline
 \ he \ & \ saw \ & \ Pete \ \\ 
 \ \ $(s/(\np\bs s))^{\bullet}$ \ \ & \ \ $((\np\bs s)/\np)^{\bullet}$ \ \ & \ \ $\np^{\bullet}$ \ \ \\ \hline
\end{tabular}
\end{center}
The formula $(s/(\np\bs s))^{\bullet}$ for \textit{he} was proposed by Lambek (\cite{lambek58}) in order to exclude occurrences in object positions. Grammaticality of the sentence under consideration w.r.t. a \textit{goal} (signed) formula $s^{\circ}$ is now established by a closed tableau 
\begin{prooftree}
\rootAtTop
\AxiomC{\tseq{p}{p:\np^{\circ},p:\np^{\bullet}}}
\AxiomC{\tseq{u}{u:s^{\circ},u:s^{\bullet}}}
\AxiomC{\tseq{v}{v:\np^{\bullet},v:\np^{\circ}}}
\RightLabel{$\beta$}
\BinaryInfC{\tseq{R_{\lgtimes}uvz}{u:s^{\circ},v:\np^{\bullet},z:(\np\bs s)^{\bullet}}}
\RightLabel{$\alpha$}
\UnaryInfC{\tseq{z}{z:(\np\bs s)^{\circ},z:(\np\bs s)^{\bullet}}}
\AxiomC{\tseq{x}{x:s^{\bullet},x:s^{\circ}}}
\RightLabel{$\beta$}
\BinaryInfC{\tseq{R_{\lgtimes}xhz}{h:(s/(\np\bs s))^{\bullet},z:(\np\bs s)^{\bullet},x:s^{\circ}}}
\RightLabel{$\beta$}
\BinaryInfC{\tseq{R_{\lgtimes}xhz,R_{\lgtimes}zwp}{h:(s/(\np\bs s))^{\bullet},w:((\np\bs s)/\np)^{\bullet},p:\np^{\bullet},x:s^{\circ}}}
\end{prooftree}
We note that, in \textbf{LG}, nothing prevents us from coupling words with output formulas. For example, the following lexicon would do just as well:
\begin{center}
\begin{tabular}{|ccc|}\hline
 \ he \ & \ saw \ & \ Pete \ \\ 
 \ \ $((\np\bs s)\os s)^{\circ}$ \ \ & \ \ $(\np\os (\np\bs s))^{\circ}$ \ \ & \ \ $\np^{\bullet}$ \ \ \\ \hline
\end{tabular}
\end{center}
as witnessed by the tableau
\begin{prooftree}
\rootAtTop
\AxiomC{\tseq{u}{u:s^{\circ},u:s^{\bullet}}}
\AxiomC{\tseq{v}{v:\np^{\bullet},v:\np^{\circ}}}
\RightLabel{$\beta$}
\BinaryInfC{\tseq{R_{\lgtimes}uvy}{u:s^{\circ},v:\np^{\bullet},y:(\np\bs s)^{\circ}}}
\AxiomC{\tseq{p}{p:\np^{\circ},p:\np^{\bullet}}}
\RightLabel{$\beta$}
\BinaryInfC{\tseq{R_{\lgtimes}uvy,R_{\lgplus}pwy}{u:s^{\circ},v:\np^{\bullet},w:(\np\os(\np\bs s))^{\circ},p:\np^{\bullet}}}
\RightLabel{$\alpha$}
\UnaryInfC{\tseq{R_{\lgplus}pwy}{y:(\np\bs s)^{\circ},w:(\np\os(\np\bs s))^{\circ},p:\np^{\bullet}}}
\AxiomC{\tseq{x}{x:s^{\bullet},x:s^{\circ}}}
\RightLabel{$\beta$}
\BinaryInfC{\tseq{R_{\lgplus}yhx,R_{\lgplus}pwy}{h:((\np\bs s)\os s)^{\circ},w:(\np\os(\np\bs s))^{\circ},p:\np^{\bullet},x:s^{\circ}}}
\end{prooftree}
\end{vb}

\section{Soundness and completeness}
Let $S$ be a finite set of labeled signed formulas and conditions $R_{\lgtimes}xyz,R_{\lgplus}xyz$. An \textit{interpretation} for $S$ is a pair $I=\langle\mathscr{M},\cdot^*\rangle$ with $\cdot^*$ a mapping of the variables occurring in $S$ to the resources of $\mathscr{M}$. \textit{Truth} w.r.t. $I$ is defined by
\begin{enumerate}
\item $R_{\delta}xyz\in S$ ($\delta\in\{\lgtimes,\lgplus\}$) is true w.r.t. $I$ in case $R_{\delta}x^*y^*z^*$ in $\mathscr{M}$.
\item $x:A^{\bullet}\in S$ is true w.r.t. $I$ if $x^*\in V(A)$ and false if $x^*\not\in V(A)$
\item $y:B^{\circ}\in S$ is true w.r.t. $I$ if $x^*\not\in V(A)$ and false if $x^*\in V(A)$
\end{enumerate}
Call $S$ \textit{satisfiable} if for some interpretation $I$, all elements of $S$ are true w.r.t. $I$.
The following observation, made w.r.t. arbitrary $I$, implies Lemma \ref{setsat}:\footnote{These conditions are easily seen to be classically equivalent to those provided in $\S$2.}
\begin{center}
\begin{tabular}{lcl}
$\alpha$ is true & $\Leftrightarrow$ & $R_{\alpha}yz$, $\alpha_1(y)$ and $\alpha_2(z)$ are true for some $y,z$ \\
$\beta$ is true & $\Leftrightarrow$ & $R_{\beta}yz$ implies $\beta_1(y)$ or $\beta_2(z)$, for arbitrary $y,z$
\end{tabular}
\end{center}
\begin{lem}\label{setsat} For any set $S$ of labeled signed formulas and conditions $R_{\lgtimes}xyz,R_{\lgplus}xyz$,
\begin{itemize}
\item[(a)] If $S$ is satisfiable and $\alpha\in S$, then for fresh $y,z$ so is $S\bigcup\{ R_{\alpha}yz,\alpha_1(y),\alpha_2(z)\}$
\item[(b)] If $S$ is satisfiable and $\beta,R_{\beta}yz\in S$, then so is $S\bigcup\{\beta_1(y)\}$ or $S\bigcup\{\beta_2(z)\}$
\end{itemize}
\end{lem}
\noindent Given a branch $\theta$ in a tableau, collect the elements of the $\Theta,\Gamma$ for each box $\tseq{\Theta}{\Gamma}$ occurring in it in a single set $S_{\theta}$ (save for when $\Theta=\{ x\}$). $\theta$ is satisfiable if $S_{\theta}$ is, and any tableau is satisfiable if one of its branches is. Lemma \ref{setsat} implies
\begin{thm}
If a tableau $\mathscr{T}$ is satisfiable, and $\mathscr{T}'$ is obtained from $\mathscr{T}$ by a single expansion, then $\mathscr{T}'$ is satisfiable.
\end{thm}
\noindent The unsatisfiability of a closed tableau is now traced to its origin. Hence, provability of $A\vdash B$ means unsatisfiability of $\tseq{\emptyset}{x:A^{\bullet},x:B^{\circ}}$, yielding soundness:
\spnewtheorem{sjonnieb}[theorem]{Corollary}{\bseries}{\itshape}
\begin{cor} All models validate provable two-formula sequents $A\vdash B$.
\end{cor}
\noindent For completeness, it suffices to show that we can simulate algebraic derivations:
\begin{center}
\begin{tabular}{cccc}
$\begin{array}{c}\infer[\textit{Refl}]{A\leq A}{}\end{array}$ \ \ & \ \ $\begin{array}{c}\infer[\textit{Trans}]{A\leq C}{A\leq B & B\leq C}\end{array}$ \ \ & \ \ $\begin{array}{c}\infer=[r]{B\leq A\bs C}{\infer=[r]{A\lgtimes B\leq C}{A\leq C/B}}\end{array}$ \ \ & \ \ $\begin{array}{c}\infer=[cr]{A\obs C\leq B}{\infer=[cr]{C\leq A\lgplus B}{C\os B\leq A}}\end{array}$
\end{tabular}
\end{center}
already shown complete in \cite{kurtoninamoortgat10}. That, for any $A$, a closed tableau of $\tseq{\emptyset}{x:A^{\bullet},x:A^{\circ}}$ exists is a simple induction on $A$'s degree. The following lemma tackles (\textit{Trans}).

\begin{lem}\label{cutelim} The following expansion (\textit{bivalence}) is admissible for closed tableaux $\mathscr{T}_1$ and $\mathscr{T}_2$ of $\tseq{\Theta}{\Gamma,\phi,\Gamma'}$ and $\tseq{\Theta'}{\Delta,\phi^{\bot},\Delta'}$
\begin{prooftree}
\rootAtTop
\AxiomC{\tseq{\Theta}{\Gamma,\phi,\Gamma'}}
\AxiomC{\tseq{\Theta'}{\Delta,\phi^{\bot},\Delta'}}
\RightLabel{B}
\BinaryInfC{\tseq{\Theta,\Theta'}{\Gamma,\Delta,\Gamma',\Delta'}}
\end{prooftree}
provided $N(\Theta)\bigcap N(\Theta')=\{ u\}$, $u$ being the label of $\phi,\phi^{\bot}$, and $\Gamma=\emptyset$ or $\Delta'=\emptyset$. 
\end{lem}
\begin{proof}
We proceed by induction on the degree of $\Gamma,\Delta,\Gamma',\Delta',\phi$.
\begin{enumerate}
\item One of $\tseq{\Theta}{\Gamma,\phi,\Gamma'}$ or $\tseq{\Theta'}{\Delta,\phi^{\bot},\Delta'}$ equals $\tseq{x}{x:p^{\bullet},x:p^{\circ}}$ or $\tseq{x}{x:p^{\circ},x:p^{\bullet}}$. Immediate, save for cases like the following, where we apply Lemma \ref{cycl}.
\begin{center}
\begin{tabular}{ccc}
$\infer[B]{\tseq{\Theta}{u:p^{\circ},\Gamma'} \ \ \ \ \tseq{u}{u:p^{\bullet},u:p^{\circ}}}{\tseq{\Theta}{\Gamma',u:p^{\circ}}}$ 
\end{tabular}
\end{center}
\item $\mathscr{T}_1$ does not close via $\phi$. Suppose $\mathscr{T}_1$ closes via $\alpha$. For example, $\Gamma=\Gamma_1,\alpha,\Gamma_1'$:
\begin{prooftree}
\rootAtTop
\AxiomC{\tseq{R_{\alpha}yz,\Theta}{\Gamma_1,\alpha_1(y),\alpha_2(z),\Gamma_1',\phi,\Gamma'}}
\RightLabel{$\alpha$}
\UnaryInfC{\tseq{\Theta}{\Gamma_1,\alpha,\Gamma_1',\phi,\Gamma'}}
\AxiomC{\tseq{\Theta'}{\Delta,\phi^{\bot},\Delta'}}
\RightLabel{B}
\BinaryInfC{\tseq{\Theta,\Theta'}{\Gamma_1,\alpha,\Gamma_1',\Delta,\Gamma',\Delta'}}
\end{prooftree}
Permuting $B$ over $\alpha$ reduces the induction measure:
\begin{prooftree}
\rootAtTop
\AxiomC{\tseq{R_{\alpha}yz,\Theta}{\Gamma_1,\alpha_1(y),\alpha_2(z),\Gamma_1',\phi,\Gamma'}}
\AxiomC{\tseq{\Theta'}{\Delta,\phi^{\bot},\Delta'}}
\RightLabel{B}
\BinaryInfC{\tseq{R_{\alpha}yz,\Theta,\Theta'}{\Gamma_1,\alpha_1(y),\alpha_2(z),\Gamma_1',\Delta,\Gamma',\Delta'}}
\RightLabel{$\alpha$}
\UnaryInfC{\tseq{\Theta,\Theta'}{\Gamma_1,\alpha,\Gamma_1',\Delta,\Gamma',\Delta'}}
\end{prooftree}
Otherwise, $\mathscr{T}_1$ closes via $\beta$. For example, $\Gamma=\Gamma_1,\Gamma_2,\beta,\Gamma_1'$ (in which case $\Delta'$ must be empty), $\Gamma'=\Gamma_1'',\Gamma_2'$ and $\Theta=\Theta_1,\Theta_2$:
\begin{prooftree}
\rootAtTop
\AxiomC{\tseq{\Theta_1}{\Gamma_1,\beta_1(y),\Gamma_1',\phi,\Gamma_1''}}
\AxiomC{\tseq{\Theta_2}{\Gamma_2,\beta_2(z),\Gamma_2'}}
\RightLabel{$\beta$}
\BinaryInfC{\tseq{R_{\beta}yz,\Theta_1,\Theta_2}{\Gamma_1,\Gamma_2,\beta,\Gamma_1',\phi,\Gamma_1'',\Gamma_2'}}
\AxiomC{\tseq{\Theta'}{\Delta,\phi^{\bot}}}
\RightLabel{B}
\BinaryInfC{\tseq{R_{\beta}yz,\Theta_1,\Theta_2,\Theta'}{\Gamma_1,\Gamma_2,\beta,\Gamma_1',\Delta,\Gamma_1'',\Gamma_2'}}
\end{prooftree}
Permuting (B) with $(\beta)$ reduces the induction measure:
\begin{prooftree}
\rootAtTop
\AxiomC{\tseq{\Theta_1}{\Gamma_1,\beta_1(y),\Gamma_1',\phi,\Gamma_1''}}
\AxiomC{\tseq{\Theta'}{\Delta,\phi^{\bot}}}
\RightLabel{B}
\BinaryInfC{\tseq{\Theta_1,\Theta'}{\Gamma_1,\beta_1(y),\Gamma_1',\Delta,\Gamma_1''}}
\AxiomC{\tseq{\Theta_2}{\Gamma_2,\beta_2(z),\Gamma_2'}}
\RightLabel{$\beta$}
\BinaryInfC{\tseq{R_{\beta}yz,\Theta_1,\Theta_2,\Theta'}{\Gamma_1,\Gamma_2,\beta,\Gamma_1',\Delta,\Gamma_1'',\Gamma_2'}}
\end{prooftree}

\item $\mathscr{T}_2$ does not close via $\phi^{\bot}$. Similar to case (2).
\item $\mathscr{T}_1$ and $\mathscr{T}_2$ close via $\phi$ and $\phi^{\bot}$ respectively. Say $\phi$ is a $\beta$, in which case $\phi^{\bot}$ is an $\alpha$. Then $\Theta=R_{\beta}yz,\Theta_1,\Theta_2$, $\Gamma=\Gamma_1,\Gamma_2$ and $\Gamma'=\Gamma_1',\Gamma_2'$: 
\begin{prooftree}
\rootAtTop
\AxiomC{\tseq{\Theta_1}{\Gamma_1,\beta_1(y),\Gamma_1'}}
\AxiomC{\tseq{\Theta_2}{\Gamma_2,\beta_2(z),\Gamma_2'}}
\RightLabel{$\beta$}
\BinaryInfC{\tseq{R_{\beta}yz,\Theta_1,\Theta_2}{\Gamma_1,\Gamma_2,\beta,\Gamma_1',\Gamma_2'}}
\AxiomC{\tseq{R_{\alpha}yz,\Theta'}{\Delta,\alpha_1(y),\alpha_2(z),\Delta'}}
\RightLabel{$\alpha$}
\UnaryInfC{\tseq{\Theta'}{\Delta,\alpha,\Delta'}}
\RightLabel{B}
\BinaryInfC{\tseq{R_{\beta}yz,\Theta_1,\Theta_2,\Theta'}{\Gamma_1,\Gamma_2,\Delta,\Gamma_1',\Gamma_2',\Delta'}}
\end{prooftree} 
We invoke the induction hypothesis twice by replacing with B-expansions on $\beta_1(u),\alpha_1(u)$ and $\beta_2(v),\alpha_2(v)$, each of lower degree: 
\begin{prooftree}
\rootAtTop
\AxiomC{\tseq{\Theta_1}{\Gamma_1,\beta_1(y),\Gamma_1'}}
\AxiomC{\tseq{\Theta_2}{\Gamma_2,\beta_2(z),\Gamma_2'}}
\AxiomC{\tseq{R_{\alpha}yz,\Theta'}{\Delta,\alpha_1(y),\alpha_2(z),\Delta'}}
\RightLabel{B}
\BinaryInfC{\tseq{R_{\alpha}yz,\Theta_2,\Theta'}{\Gamma_2,\Delta,\alpha_1(y),\Gamma_2',\Delta'}}
\RightLabel{B}
\BinaryInfC{\tseq{R_{\alpha}yz,\Theta_1,\Theta_2,\Theta'}{\Gamma,\Gamma_2,\Delta,\Gamma_1',\Gamma_2',\Delta'}}
\end{prooftree} 
\end{enumerate}
\end{proof}
\noindent Simulation of transitivity immediately follows. Lemma \ref{cutelim} also applies in showing (co)residuation derivable. For example, suppose we have a closed tableau of $\tseq{z}{z:C^{\bullet},z:(A\lgplus B)^{\circ}}$. Then for some (fresh) $y$, $\tseq{y}{y:(A\obs C)^{\bullet},y:B^{\circ}}$ also closes:
\begin{prooftree}
\rootAtTop
\AxiomC{\tseq{y}{y:B^{\bullet},y:B^{\circ}}}
\AxiomC{\tseq{x}{x:A^{\circ},x:A^{\bullet}}}
\RightLabel{$\beta$}
\BinaryInfC{\tseq{R_{\lgplus}zxy}{x:A^{\circ},z:(A\lgplus B)^{\bullet},y:B^{\circ}}}
\AxiomC{\tseq{z}{z:C^{\bullet},z:(A\lgplus B)^{\circ}}}
\RightLabel{B}
\BinaryInfC{\tseq{R_{\lgplus}zxy}{x:A^{\circ},z:C^{\bullet},y:B^{\circ}}}
\RightLabel{$\alpha$}
\UnaryInfC{\tseq{y}{y:(A\obs C)^{\bullet},y:B^{\circ}}}
\end{prooftree}
The above observations imply\begin{thm} The tableau method for \textbf{LG} is complete.
\end{thm}
\noindent The following is now an easy consequence of the subformula property:
\begin{cor}\label{conserv}
\textbf{LG} conservatively extends \textbf{NL}.
\end{cor}

\section{Lambek-Grishin grammars are context-free}
We use our tableau method to establish context-freeness of Lambek-Grishin grammars. Following the strategy laid out in \cite{pentus99} and \cite{jager04}, we rely on an interpolation property proven in Lemma \ref{interpol}.

By an \textbf{LG} \textit{grammar} $\mathscr{G}$ we shall understand a tuple $\langle\mathscr{A},L,g^{\Diamond}\rangle$ consisting of: a set of \textit{words} $\mathscr{A}$; a \textit{lexicon} $L$ mapping words to (finite) sets of signed (!) formulas; and a signed atomic \textit{goal} formula $g^{\Diamond}$ ($\Diamond\in\{\bullet,\circ\}$). The \textit{language} $\mathscr{L}(\mathscr{G})$ recognized by $\mathscr{G}$ we then define by the set of lists $w_1,\dots,w_n$ of words $w_i\in\mathscr{A}$ ($1\leq i\leq n$) such that, for some $A_1^{\Diamond 1}\in L(w_1),\dots, A_n^{\Diamond n}\in L(w_n)$ ($\Diamond 1,\dots,\Diamond n\in\{\bullet,\circ\}$) and tree $\Theta$, $\tseq{\Theta}{x_1:A_1^{\Diamond 1},\dots,x_n:A_n^{\Diamond n},x:g^{\Diamond}}$ closes.

We proceed to show context-freeness of \textbf{LG} grammars. Recognizability of context-free languages is a consequence of Kandulski's results for \textbf{NL} and Corollary \ref{conserv}. Our strategy for showing that every \textbf{LG} grammar also has an equivalent  context-free grammar follows closely that of \cite{jager04}, inspired in turn by \cite{pentus99}. We first prove an interpolation property for our tableaux. 

\begin{lem}\label{interpol} 
Suppose $\tseq{\Theta,\Theta'}{\Gamma,\Delta,\Gamma'}$ closes s.t. $N(\Theta)\bigcap N(\Theta')=\{ u\}$, and the variables in $\Gamma$ and $\Gamma'$ ($\Delta$) draw from $\Theta$ ($\Theta'$). Then for some $\phi=u:C^{\Diamond}$, with $\Diamond\in\{\bullet,\circ\}$ depending on whether $u\in H(\Theta)$ or $u\in C(\Theta)$, and with $C$ a subformula of (a formula in) $\Gamma,\Delta,\Gamma'$, $\tseq{\Theta}{\Gamma,\phi,\Gamma'}$ and $\tseq{\Theta'}{\Delta,\phi^{\bot}}$ close.
\end{lem}
\begin{proof}
We refer to $\phi$ and $C$ interchangeably as the \textit{witness} for $\Delta$ (borrowing terminology from \cite{jager04}). We proceed by induction on the degree of $\Gamma,\Delta,\Gamma'$. If $S=\tseq{\Theta,\Theta'}{\Gamma,\Delta,\Gamma'}$ is already of the form $\tseq{u}{u:p^{\bullet},u:p^{\circ}}$ or $\tseq{u}{u:p^{\circ},u:p^{\bullet}}$, take $C=p$. Otherwise, the tableau $\mathscr{T}$ for $S$ closes via some $\psi$ in $\Gamma$, $\Gamma'$ or $\Delta$. 
\begin{enumerate}
\item $\phi\in\Delta$. If $\psi=\alpha$, $\Delta=\Delta_1,\alpha,\Delta_2$ with $\mathscr{T}$ taking the form
\begin{prooftree}
\rootAtTop
\AxiomC{\tseq{R_{\alpha}yz,\Theta,\Theta'}{\Gamma,\Delta_1,\alpha_1(y),\alpha_2(z),\Delta_2,\Gamma'}}
\RightLabel{$\alpha$}
\UnaryInfC{\tseq{\Theta,\Theta'}{\Gamma,\Delta_1,\alpha,\Delta_2,\Gamma'}}
\end{prooftree}
Apply the induction hypothesis to obtain a witness $\omega$ for $\Delta_1,\alpha_1(y),\alpha_2(z),\Delta_2$, i.e.,  $\tseq{\Theta}{\Gamma,\omega,\Gamma'}$ and $\tseq{R_{\alpha}yz,\Theta'}{\Delta_1,\alpha_1(y),\alpha_2(z),\Delta_2,\omega^{\bot}}$ close. We can take $\phi=\omega$. Indeed, we obtain a closed tableau for $\tseq{\Theta'}{\Delta_1,\alpha,\Delta_2,\omega^{\bot}}$ by an $\alpha$-expansion:
\begin{prooftree}
\rootAtTop
\AxiomC{\tseq{R_{\alpha}yz,\Theta'}{\Delta_1,\alpha_1(y),\alpha_2(z),\Delta_2,\omega^{\bot}}}
\RightLabel{$\alpha$}
\UnaryInfC{\tseq{\Theta'}{\Delta_1,\alpha,\Delta_2,\omega^{\bot}}}
\end{prooftree}
If $\psi=\beta$, we must consider two subcases. If $\Delta=\psi$, then $\Theta'=\{ u\}$ and we may take $\phi=\psi^{\bot}$. Otherwise, $\Theta'=R_{\beta}yz,\Theta_1,\Theta_2$ ($N(\Theta_1)\bigcap N(\Theta_2)=\emptyset$, $N(\Theta_1)\bigcap N(R_{\beta}yz)=y$ and $N(\Theta_2)\bigcap N(R_{\beta}yz)=\{ z\}$), and either $u\in N(\Theta_1)$ or $u\in N(\Theta_2)$. In the former case, $\Delta=\Delta_1,\Delta_2,\beta,\Delta_3$ with $\mathscr{T}$ taking the form
\begin{prooftree}
\rootAtTop
\AxiomC{\tseq{\Theta,\Theta_1}{\Gamma,\Delta_1,\beta_1(y),\Delta_3,\Gamma'}}
\AxiomC{\tseq{\Theta_2}{\Delta_2,\beta_2(z)}}
\RightLabel{$\beta$}
\BinaryInfC{\tseq{R_{\beta}yz,\Theta,\Theta_1,\Theta_2}{\Gamma,\Delta_1,\Delta_2,\beta,\Delta_3,\Gamma'}}
\end{prooftree}
Apply the induction hypothesis to find a witness $\omega$ for $\Delta_1,\beta_1(y),\Delta_3$, i.e., so that $\tseq{\Theta}{\Gamma,\omega,\Gamma'}$ and $\tseq{\Theta_1}{\Delta_1,\beta_1(y),\Delta_3,\omega^{\bot}}$ close. We may take $\phi=\omega$. Indeed, a closed tableau for $\tseq{\Theta_1,\Theta_2}{\Delta_1,\Delta_2,\beta_2(z),\Delta_3,\omega^{\bot}}$ is found after a $\beta$-expansion:
\begin{prooftree}
\rootAtTop
\AxiomC{\tseq{\Theta_1}{\Delta_1,\beta_1(y),\Delta_3,\omega^{\bot}}}
\AxiomC{\tseq{\Theta_2}{\Delta_2,\beta_2(z)}}
\RightLabel{$\beta$}
\BinaryInfC{\tseq{\Theta_1,\Theta_2}{\Delta_1,\Delta_2,\beta_2(z),\Delta_3,\omega^{\bot}}}
\end{prooftree}
If instead $u\in N(\Theta_2)$, then $\Delta=\Delta_1,\beta,\Delta_2,\Delta_3$ and $\mathscr{T}$ takes the form
\begin{prooftree}
\rootAtTop
\AxiomC{\tseq{\Theta_1}{\beta_1(y),\Delta_2}}
\AxiomC{\tseq{\Theta,\Theta_2}{\Gamma,\Delta_1,\beta_2(z),\Delta_3,\Gamma'}}
\RightLabel{$\beta$}
\BinaryInfC{\tseq{\Theta_1,\Theta_2}{\Gamma,\Delta_1,\beta_2(z),\Delta_2,\Delta_3,\Gamma'}}
\end{prooftree}
This time, apply the induction hypothesis on $\tseq{\Theta,\Theta_2}{\Gamma,\Delta_1,\beta_2(z),\Delta_3,\Gamma'}$.

\item $\psi\in\Gamma$. The case where $\psi=\alpha$ is easy (similar to when $\psi=\alpha\in\Delta$). So suppose $\psi=\beta$. If $\Gamma=\psi$, then $\Theta'=R_{\beta}yz,\Theta_1,\Theta_2$, $\Delta=\Delta_1,\Delta_2$ with $\mathscr{T}$ taking the form
\begin{prooftree}
\rootAtTop
\AxiomC{\tseq{\Theta_1}{\beta_1(y),\Delta_1}}
\AxiomC{\tseq{\Theta_2}{\beta_2(z),\Delta_2}}
\RightLabel{$\beta$}
\BinaryInfC{\tseq{R_{\beta}yz,\Theta_1,\Theta_2}{\beta,\Delta_1,\Delta_2}}
\end{prooftree}
Note that $\Gamma'=\emptyset$ as $u\not\in N(\Theta_1)\bigcup N(\Theta_2)$, so also $\Theta=\{ u\}$. Evidently, we may take $\phi=\psi$. Now suppose $\Theta=R_{\beta}yz,\Theta_1,\Theta_2$ with $N(\Theta_1)\bigcap N(\Theta_2)=\emptyset$, $N(\Theta_1)\bigcap N(R_{\beta}yz)=\{ y\}$ and $N(\Theta_2)\bigcap N(R_{\beta}yz)=\{ z\}$. We must consider the cases $u\in N(\Theta_1)$ and $u\in N(\Theta_2)$. We consider the former, the latter being handled similarly (although then $\Gamma'=\emptyset$). Now $\Gamma=\Gamma_1,\Gamma_2,\beta,\Gamma_3$ and $\Gamma'=\Gamma_1',\Gamma_2'$ such that $\mathscr{T}$ takes the form
\begin{prooftree}
\rootAtTop
\AxiomC{\tseq{\Theta_1,\Theta}{\Gamma_1,\beta_1(y),\Gamma_3,\Delta,\Gamma_1'}}
\AxiomC{\tseq{\Theta_2}{\Gamma_2,\beta_2(z),\Gamma_2'}}
\RightLabel{$\beta$}
\BinaryInfC{\tseq{R_{\beta}yz,\Theta_1,\Theta_2,\Theta'}{\Gamma_1,\Gamma_2,\beta,\Gamma_3,\Delta,\Gamma_1',\Gamma_2'}}
\end{prooftree}
and we may apply the induction hypothesis on $\tseq{\Theta_1,\Theta}{\Gamma_1,\beta_1(y),\Gamma_3,\Delta,\Gamma_1'}$ to find an $\omega$ for which $\tseq{\Theta_1}{\Gamma_1,\beta_1(y),\Gamma_3,\omega,\Gamma_1'}$ and $\tseq{\Theta'}{\Delta,\omega^{\bot}}$ close. We take $\phi=\omega$. Indeed, we find a closed tableau for $\tseq{R_{\beta}yz,\Theta_1,\Theta_2}{\Gamma,\omega,\Gamma'}$ as follows:
\begin{prooftree}
\rootAtTop
\AxiomC{\tseq{\Theta_1}{\Gamma_1,\beta_1(y),\Gamma_3,\omega,\Gamma_1'}}
\AxiomC{\tseq{\Theta_2}{\Gamma_2,\beta_2(z),\Gamma_2'}}
\RightLabel{$\beta$}
\BinaryInfC{\tseq{R_{\beta}yz,\Theta_1,\Theta_2}{\Gamma_1,\Gamma_2,\beta,\Gamma_3,\omega,\Gamma_1',\Gamma_2'}}
\end{prooftree}

\item $\phi\in\Gamma'$. Similar to the previous case.

\end{enumerate}

\end{proof}
\begin{lem}\label{axiomat}
For $T$ a set of formulas, closed under taking subformulas, define
\begin{center}
\begin{tabular}{ccl}
$\textbf{LG}_T$ & $=_{\textit{def}}$ & $\{ S=\tseq{x}{x:A^{\bullet},B^{\circ}} \ | \ A,B\in T \ \& \ S \ \textrm{closes}\}$ \\
 & $\bigcup$ & $\{ S=\tseq{x}{x:B^{\circ},x:A^{\bullet}} \ | \ A,B\in T \ \& \ S \ \textrm{closes}\}$ \\
 & $\bigcup$ & $\{ S=\tseq{R_{\lgtimes}xyz}{y:A^{\bullet},z:B^{\bullet},x:C^{\circ}} \ | \ A,B\in T \ \& \ S \ \textrm{closes}\}$ \\
 & $\bigcup$ & $\{ S=\tseq{R_{\lgtimes}xyz}{x:C^{\circ},y:A^{\bullet},z:B^{\bullet}} \ | \ A,B\in T \ \& \ S \ \textrm{closes}\}$ \\
 & $\bigcup$ & $\{ S=\tseq{R_{\lgtimes}xyz}{z:B^{\bullet},x:C^{\circ},y:A^{\bullet}} \ | \ A,B\in T \ \& \ S \ \textrm{closes}\}$ \\
 & $\bigcup$ & $\{ S=\tseq{R_{\lgplus}xyz}{y:A^{\circ},z:B^{\circ},x:C^{\bullet}} \ | \ A,B\in T \ \& \ S \ \textrm{closes}\}$ \\
 & $\bigcup$ & $\{ S=\tseq{R_{\lgplus}xyz}{x:C^{\bullet},y:A^{\circ},z:B^{\circ}} \ | \ A,B\in T \ \& \ S \ \textrm{closes}\}$ \\
 & $\bigcup$ & $\{ S=\tseq{R_{\lgplus}xyz}{z:B^{\circ},x:C^{\bullet},y:A^{\circ}} \ | \ A,B\in T \ \& \ S \ \textrm{closes}\}$
\end{tabular}
\end{center}
Now suppose $\tseq{\Theta}{\Gamma}$ closes with all formulas of $\Gamma$ in $T$. Then $\tseq{\Theta}{\Gamma}$ has a tableau whose branches end in members of $\textbf{LG}_T$ and with the following instance of bivalence (B) as the sole type of expansion, provided $\Gamma'$ is not empty.
\begin{prooftree}
\rootAtTop
\AxiomC{\tseq{\Theta}{\Gamma,\phi,\Gamma'}}
\AxiomC{\tseq{\Theta'}{\Delta,\phi^{\bot}}}
\RightLabel{B}
\BinaryInfC{\tseq{\Theta,\Theta'}{\Gamma,\Delta,\Gamma'}}
\end{prooftree}
\end{lem}
\begin{proof}
By induction on the cardinality of $\Theta$. In the base case, $\Theta$ equals $\{ x\}$, $\{ R_{\lgtimes}xyz\}$ or $\{ R_{\lgplus}xyz\}$, and $\tseq{\Theta}{\Gamma}\in\textbf{LG}_T$ by definition. Now suppose $\Theta=\Theta_1,\Theta_2$, $N(\Theta_1)\bigcap N(\Theta_2)=\{ u\}$, both $\Theta_1\not =\{ u\}$ and $\Theta_2\not =\{ u\}$, $\Gamma=\Gamma_1,\Gamma_2,\Gamma_3$ and the variables in $\Gamma_1$ and $\Gamma_3$ ($\Gamma_2$) draw from $\Theta_1$ ($\Theta_2$). By Lemma \ref{interpol}, there now exists $\phi$ with label $u$ s.t. $\tseq{\Theta_1}{\Gamma_1,\phi,\Gamma_3}$ and $\tseq{\Theta_2}{\Gamma_2,\phi^{\bot}}$ close. We can assume $\Gamma_3$ is not empty, as otherwise we could have picked $\Gamma_1$ for instantiating $\Delta$ in (B) as opposed to $\Gamma_2$. Since the cardinalities of $\Theta_1$ and $\Theta_2$ are strictly smaller than that of $\Theta_1,\Theta_2$ (as $\Theta_2\not =\{ u\}$), the induction hypothesis applies to $\tseq{\Theta_1}{\Gamma_1,\phi,\Gamma_3}$ and $\tseq{\Theta_2}{\Gamma_2,\phi^{\bot}}$. The statement of the lemma obtains after another application of (B).
\end{proof}
\begin{thm}
For every \textbf{LG}-grammar $\mathscr{G}$, $\mathscr{L}(\mathscr{G})$ is context-free.
\end{thm}
\begin{proof}
Suppose we have an \textbf{LG}-grammar $\mathscr{G}_1=\langle\mathscr{A},L,g^{\Diamond}\rangle$. Refer by $T$ to the set of formulas in the range of $L$, closed under taking subformulas. We now construct the following context-free grammar $\mathscr{G}_2$: its set of terminals coincides with $\mathscr{A}$; its nonterminals are specified by $\{ A^{\bullet} \ | \ A\in T\}\bigcup\{ A^{\circ} \ | \ A\in T\}$; its start symbol is $g^{\Diamond\bot}$ for $\Diamond\bot=\bullet$ if $\Diamond=\circ$ and $\Diamond\bot=\circ$ if $\Diamond=\bullet$; and its productions are given by
\begin{center}
\begin{tabular}{cl}
 & $\{ B^{\bullet}\rightarrow A^{\bullet} \ | \ \tseq{x}{x:A^{\bullet},x:B^{\circ}}\in\textbf{LG}_T\}$ \\
$\bigcup$ & $\{ A^{\circ}\rightarrow B^{\circ} \ | \ \tseq{x}{x:B^{\circ},x:A^{\bullet}}\in\textbf{LG}_T\}$ \\
$\bigcup$ & $\{ C^{\bullet}\rightarrow A^{\bullet},B^{\bullet} \ | \ \tseq{R_{\lgtimes}xyz}{y:A^{\bullet},z:B^{\bullet},x:C^{\circ}}\in\textbf{LG}_T\}$ \\
$\bigcup$ & $\{ B^{\circ}\rightarrow C^{\circ},A^{\bullet} \ | \ \tseq{R_{\lgtimes}xyz}{x:C^{\circ},y:A^{\bullet},z:B^{\bullet}}\in\textbf{LG}_T\}$ \\
$\bigcup$ & $\{ A^{\circ}\rightarrow B^{\bullet},C^{\circ} \ | \ \tseq{R_{\lgtimes}xyz}{z:B^{\bullet},x:C^{\circ},y:A^{\bullet}}\in\textbf{LG}_T\}$ \\
$\bigcup$ & $\{ C^{\circ}\rightarrow A^{\circ},B^{\circ} \ | \ \tseq{R_{\lgplus}xyz}{y:A^{\circ},z:B^{\circ},x:C^{\bullet}}\in\textbf{LG}_T\}$ \\
$\bigcup$ & $\{ B^{\bullet}\rightarrow C^{\bullet},A^{\circ} \ | \ \tseq{R_{\lgplus}xyz}{x:C^{\bullet},y:A^{\circ},z:B^{\circ}}\in\textbf{LG}_T\}$ \\
$\bigcup$ & $\{ A^{\bullet}\rightarrow B^{\circ},C^{\bullet} \ | \ \tseq{R_{\lgplus}xyz}{z:B^{\circ},x:C^{\bullet},y:A^{\circ}}\in\textbf{LG}_T\}$ \\
$\bigcup$ & $\{ A^{\bullet}\rightarrow w \ | \ w\in\mathscr{A}, \ A^{\bullet}\in L(w)\}$ \\
$\bigcup$ & $\{ A^{\circ}\rightarrow w \ | \ w\in\mathscr{A}, \ A^{\circ}\in L(w)\}$
\end{tabular}
\end{center}
We claim $\mathscr{G}_1$ and $\mathscr{G}_2$ recognize the same languages. 
\begin{itemize}
\item Going from left to right, assume $\mathscr{G}_1$ recognizes $w_1,\dots,w_n$. Then for some $A_1^{\Diamond 1}\in L(w_1),\dots,A_n^{\Diamond n}\in L(w_n)$ and $\Theta$, $S=\tseq{\Theta}{x_1:A_1^{\Diamond 1},\dots,x_n:A_n^{\Diamond n},x:g^{\Diamond}}$ closes. We claim $g^{\Diamond\bot}\rightarrow^* A_1^{\Diamond 1},\dots,A_n^{\Diamond n}$, and hence $g^{\Diamond\bot}\rightarrow^* w_1,\dots,w_n$. This follows from an inductive argument on the tableau of $S$ constructed by Lemma \ref{axiomat}, proving that if $\tseq{\Theta}{y_1:B_1^{\Diamond'1},\dots,y_m:B_m^{\Diamond'm},y:B^{\Diamond'}}$ closes, then $B^{\Diamond'\bot}\rightarrow^*B_1^{\Diamond'1},\dots,B_m^{\Diamond'm}$. The base cases follow from the construction of $\mathscr{G}_2$, while the sole inductive case depends on the transitive closure of $\rightarrow^*$. 

\item Conversely, suppose $g^{\Diamond\bot}\rightarrow^* w_1,\dots,w_n$. Then, by the construction of $\mathscr{G}_2$, $g^{\Diamond\bot}\rightarrow^* A_1^{\Diamond 1},\dots,A_n^{\Diamond n}$ for some $A_1^{\Diamond 1}\in L(w_1),\dots,A_n^{\Diamond n}\in L(w_n)$. Since all production rules involved draw from elements of $\textbf{LG}_T$, a straightforward inductive argument constructs a closed tableau of $\tseq{\Theta}{x_1:A_1^{\Diamond 1},\dots,x_n:A_n^{\Diamond n},x:g^{\Diamond}}$ for some $\Theta$ using $B$-expansions, and we remove the latter one by one from bottom to top through repeated applications of Lemma \ref{cutelim}.
\end{itemize}
\end{proof}
\noindent In \cite{moortgat09}, a slightly different notion of \textbf{LG}-grammars is used. Stated as a special case of our grammars $\mathscr{G}=\langle\mathscr{A},L,g^{\Diamond}\rangle$, $\Diamond$ is fixed at $\circ$ and the range of $L$ is restricted to signed formulas $A^{\bullet}$. Moreover, the language $\mathscr{L}(\mathscr{G})$ recognized by $\mathscr{G}$ now reads as the set of lists $w_1,\dots,w_n$ of words s.t. for some $A_1^{\bullet}\in L(w_1),\dots, A_n^{\bullet}\in L(w_n)$ and tree $\Theta$, $S=\tseq{\Theta}{x_1:A_1^{\bullet},\dots,x_n:A_n^{\bullet},x:g^{\circ}}$ closes, provided $\Theta$ lacks conditions of the form $R_{\lgplus}xyz$. Note, though, that the latter kind of conditions may still appear further down in the tableau for $S$. In particular, $A_1,\dots,A_n$ may freely contain connectives from the coresiduated family $\{\lgplus,\os,\obs\}$. Seeing as the above definitions constitute special cases of ours, context-freeness is preserved.

\subsubsection{Acknowledgements.} This work has benefited from discussions with Michael Moortgat, Jeroen Bransen and Vincent van Oostrom, as well as from comments from an anonymous referee. All remaining errors are my own.
\bibliography{bibliography}

\end{document}